\pdfoutput=1

\documentclass[11pt]{article}

\usepackage[]{emnlp2021}

\usepackage{times}
\usepackage{latexsym}

\usepackage[T1]{fontenc}

\usepackage[utf8]{inputenc}

\usepackage{microtype}

\usepackage{amsfonts,amsmath,amsthm}
\usepackage{subcaption}
\usepackage{booktabs}
\usepackage{float}
\usepackage{tikz,pgfplots}
\usetikzlibrary{positioning,arrows.meta}
\newtheorem{proposition}{Proposition}
\newtheorem{proposition_appendix}{Proposition}[section]
\newtheorem{lemma}{Lemma}[section]
\newtheorem{property}{Property}

\newcommand{\name}{PermuteFormer}

\title{\name{}: Efficient Relative Position Encoding for Long Sequences}

\author{
  Peng Chen \\
  Peking University \\
  \texttt{chen.peng@pku.edu.cn} 
 }

\begin{document}
\maketitle
\begin{abstract}
A recent variation of Transformer, Performer, scales Transformer to longer sequences with a linear attention mechanism.
However, it is not compatible with relative position encoding, which has advantages over absolute position encoding.
In this paper, we discuss possible ways to add relative position encoding to Performer. Based on the analysis, we propose \name{}, a Performer-based model with relative position encoding that scales linearly on long sequences.
\name{} applies position-dependent transformation on queries and keys to encode positional information into the attention module. This transformation is carefully crafted so that the final output of self-attention is not affected by absolute positions of tokens.
\name{} introduces negligible computational overhead by design that it runs as fast as Performer.
We evaluate \name{} on Long-Range Arena, a dataset for long sequences, as well as WikiText-103, a language modeling dataset. The experiments show that \name{} uniformly improves the performance of Performer with almost no computational overhead and outperforms vanilla Transformer on most of the tasks. \footnote{Code is available at \url{https://github.com/cpcp1998/PermuteFormer}.}
\end{abstract}

\section{Introduction}

The Transformer architecture \citep{Vaswani:17} has achieved state-of-the-art on various fields of research, including natural language processing \citep{devlin-etal-2019-bert, Raffel20}, speech processing \citep{Baevski:20} and image processing \citep{Dosovitskiy:20, tan-bansal-2019-lxmert}. But Transformer does not scale well to long sequences, because the time complexity and memory complexity of the attention module in Transformer are both quadratic to the sequence length.
Recently, several efficient Transformers \citep{Kitaev:20, Wang:20, Zaheer:20, Xiong:21} have been proposed to speed up the model from quadratic complexity to linear complexity without significant performance loss.
Generally, they utilize efficient algorithms to approximate attention.
\S\ref{sec:effcient_transformers} briefly introduces these efficient Transformers and a more thorough review can be found in \citet{Tay:20b}.

Among these efficient Transformers, it is suggested that Performer \citep{Choromanski:20} is the fastest one \citep{Tay:20}. In this paper, we denote as Performer the family of efficient Transformers similar to \citet{Choromanski:20}, e.g., \citet{Katharopoulos:20, Peng21, Kasai:21, Likhosherstov:20}, not only \citet{Choromanski:20} itself. Performer utilizes kernel method to avoid explicit calculation of attention weights. It applies a non-linear feature map to queries and keys to get query features and key features respectively and then multiplies query features, key features, and values together directly, without applying softmax. With the appropriate ordering of matrix multiplications, Performer achieves complexity linear of the sequence length. Moreover, some implementation of unidirectional Performer \citep{Likhosherstov:20} even reduces memory footprint to constant at both training time and inference time.

Although Performer accelerates attention  to linear complexity, the existing relative position encoding \citep{shaw-etal-2018-self, dai-etal-2019-transformer, Raffel20} still has quadratic complexity with respect to the sequence length. So Performer cannot benefit from relative position encoding, which has already been a common practice for a bunch of state-of-the-art Transformers \citep{Yang19, Raffel20, He20}. Relative position encoding has several advantages over absolute position encoding. (1) Relative position encoding may be applied to sequences with arbitrary lengths, with no limitation imposed by training datasets. (2) Relative position encoding is more efficient and effective than absolute position encoding. \citep{shaw-etal-2018-self}

Besides Performer, existing relative position encodings also do not fit with other efficient Transformers. Some relative position encoding \citep{Raffel20} adds a bias to the attention matrix, and others \citep{shaw-etal-2018-self, dai-etal-2019-transformer} add a relative-position-dependent bias to key vectors. Both require explicit calculation of dot-products between query vectors and key vectors. This conflicts with the second and third categories of efficient Transformers described in Section 2 because they reduce the computation complexity by avoiding the explicit calculation of dot-products between query vectors and key vectors. As for the first category of efficient Transformers, LSH in \citet{Kitaev:20} may fail to locate major attention weights in the presence of relative position encoding; \citet{Zaheer:20, Beltagy:20} rely on global tokens heavily, whose relative positions to other tokens are not defined.

In this paper, we propose a Performer-compatible relative position encoding that scales linearly on long sequences. Performer with this novel relative position encoding is named \name{}.
\name{} applies a position-aware transformation on query features and key features to encode positional information. More specifically, we choose a random permutation $\pi:\{1,2,\cdots,d\}\to\{1,2,\cdots,d\}$ where $d$ is the dimension of query / key features per attention head, and applies the permutation $i$ times to $i$-th token's query / key feature.%
\footnote{When we say applying a permutation $\pi$ to a vector $\mathbf{x}=[x_1, x_2, \cdots, x_d]$, we mean the operation maps $\mathbf{x}$ to vector $[x_{\pi(1)}, x_{\pi(2)}, \cdots, x_{\pi(d)}]$.}
In this way, positional information is encoded into attention weights.
We prove that, although the transformation applied to query feature and key feature of a token depends on its absolute position, the effects of absolute position on query features and key features cancel out with each other on calculating dot-product of them. Thus, the final attention weights do not depend on the absolute positions, and \name{} encodes relative position only.

\name{} is as efficient as Performer, with negligible computational overhead. Permuting of query features and key features can be implemented efficiently, with computational complexity proportional to their size. Since the size is far less than the computational complexity of the whole model, the cost of permutation in \name{} is negligible compared to the overall computational cost of Performer.
The analysis above is also confirmed by the experiment results.

We evaluate \name{} on Long-Range Arena \citep{Tay:20} for bidirectional case and on WikiText-103 \citep{Merity17} for unidirectional case. Long-Range Arena is a benchmark designed to evaluate efficient Transformers on long sequences. We find that the new relative position encoding improves the performance of \name{} significantly on Long-Range Arena. It not only performs better than Performer but also out-performs the vanilla Transformer, as well as other efficient Transformers, e.g., \citet{Kitaev:20, Wang:20, Xiong:21}. WikiText-103 is a language modeling dataset. \name{} reduces the performance gap between Performer and Transformer on WikiText-103. It also speeds up the convergence of the model.

\paragraph{Contributions}

The main contribution of this paper is summarized as follows.

\begin{itemize}
    \item We discuss possible ways to add relative position encoding to Performer. We theoretically propose three properties that Performer-compatible relative position encoding should hold.
    \item We introduce \name{}, a Performer model with relative position encoding that scales linearly to long sequences. It permutes elements of query features and key features to encode positional information. It is the only Performer-compatible relative position encoding with linear complexity, as far as we know. \name{} is as efficient as Performer.
    \item We conduct extensive experiments to evaluate \name{}. It achieves strong empirical performance and obtains state-of-the-art on Long-Range Arena, a benchmark for efficient Transformers. It also improves the performance of Performer on language modeling tasks like WikiText-103.
\end{itemize}

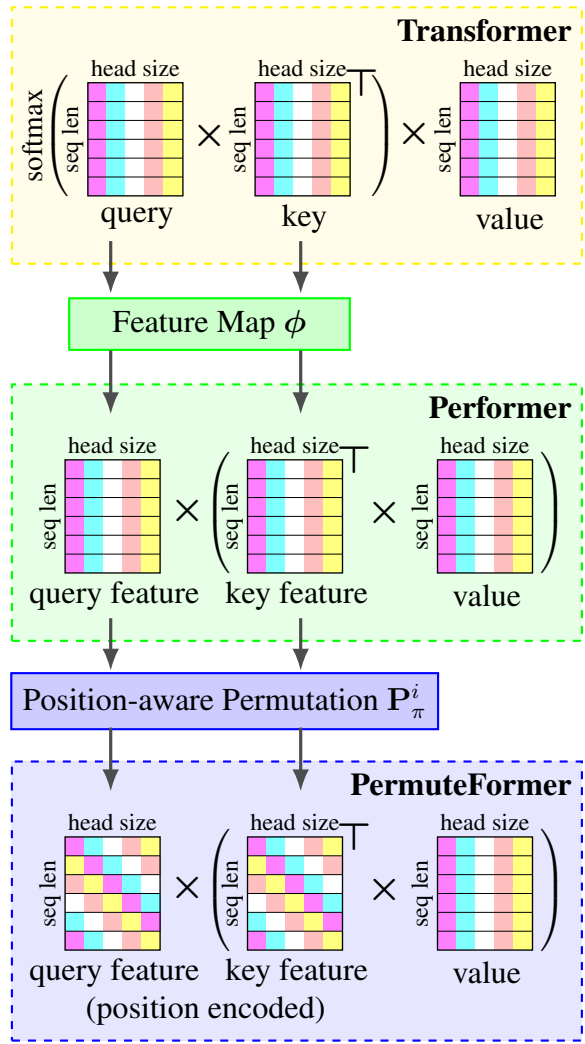
\begin{figure}
    \centering
    \def\colorPallete{{
        "1.00 1.00 0.50",
        "1.00 0.50 1.00",
        "0.50 1.00 1.00",
        "1.00 1.00 1.00",
        "1.00 0.75 0.75",
        "0.75 1.00 0.75",
        "0.75 0.75 1.00",
    }}
    \def\vector(#1:#2:#3:#4:#5:#6) 
    {
        \foreach\y in {1,...,#4} {
            \foreach\x in {1,...,#3} {
                \pgfmathparse{\colorPallete[Mod(\x+\y*#6-#4*#6,#3)]};
                \definecolor{currentColor}{rgb}{\pgfmathresult};
                \fill[currentColor] (#1+\x*#5-#5,#2+\y*#5-#5) rectangle (#1+\x*#5,#2+\y*#5) ;
            }
        }
        \draw[line width=0.5pt] (#1,#2) -- (#1,#2+#4*#5) -- (#1+#3*#5,#2+#4*#5) -- (#1+#3*#5,#2) -- cycle ;
        \foreach\y in {2,3,...,#4} {\draw[line width=0.3pt] (#1,#2+\y*#5-#5) -- (#1+#3*#5,#2+\y*#5-#5) ;}
        \node at (#1+#3*#5/2,#2+#4*#5+0.2) {\small head size};
        \node[rotate=90] at (#1-0.2,#2+#4*#5/2) {\small seq len};
    }
    \begin{tikzpicture}
        \begin{scope}[shift={(0,10.5)}]
            \fill[yellow!10!white] (0, -0.2) rectangle (7.5, 3.2);
            \draw[yellow, dashed, thick] (0, -0.2) -- (7.5, -0.2) -- (7.5, 3.2) -- (0, 3.2) -- cycle;
            \node[anchor=east] at (7.5, 2.9) {\large \textbf{Transformer}} ;
            \node[rotate=90, anchor=north] at (0, 1.5) {softmax} ;
            \vector(1:0.7:5:6:0.25:0)
            \node at (2.6,1.5) {\Large $\boldsymbol\times$} ;
            \vector(3.2:0.7:5:6:0.25:0)
            \node at (0.6, 1.5) {$\left(\vphantom{\rule{0cm}{0.9cm}}\right.$} ;
            \node at (5.1, 1.5) {$\left.\vphantom{\rule{0cm}{0.9cm}}\right)$\Large $\boldsymbol\times$} ;
            \node at (4.6, 2.2) {\Large$\boldsymbol\top$};
            \vector(5.9:0.7:5:6:0.25:0)
            \node at (1.65,0.4) {\large query};
            \node at (3.85,0.4) {\large key};
            \node at (6.55,0.4) {\large value};
        \end{scope}
        \begin{scope}[shift={(0.3,5.5)}]
            \fill[green!10!white] (-0.3, -0.2) rectangle (7.2, 3.2);
            \draw[green, dashed, thick] (-0.3, -0.2) -- (7.2, -0.2) -- (7.2, 3.2) -- (-0.3, 3.2) -- cycle;
            \node[anchor=east] at (7.2, 2.9) {\large \textbf{Performer}} ;
            \vector(0.4:0.7:5:6:0.25:0)
            \vector(2.8:0.7:5:6:0.25:0)
            \node at (2.2, 1.5) {\Large $\boldsymbol\times$\normalsize$\left(\vphantom{\rule{0cm}{0.9cm}}\right.$} ;
            \node at (6.8, 1.5) {$\left.\vphantom{\rule{0cm}{0.9cm}}\right)$} ;
            \node at (4.6,1.5) {\Large $\boldsymbol\times$} ;
            \node at (4.2, 2.2) {\Large$\boldsymbol\top$};
            \vector(5.3:0.7:5:6:0.25:0)
            \node at (1.05,0.4) {\large query feature};
            \node at (3.45,0.4) {\large key feature};
            \node at (5.95,0.4) {\large value};
        \end{scope}
        \begin{scope}[shift={(0.3,0.5)}]
            \fill[blue!10!white] (-0.3, -0.5) rectangle (7.2, 3.2);
            \draw[blue, dashed, thick] (-0.3, -0.5) -- (7.2, -0.5) -- (7.2, 3.2) -- (-0.3, 3.2) -- cycle;
            \node[anchor=east] at (7.2, 2.9) {\large \textbf{\name{}}} ;
            \vector(0.4:0.7:5:6:0.25:1)
            \vector(2.8:0.7:5:6:0.25:1)
            \node at (2.2, 1.5) {\Large $\boldsymbol\times$\normalsize$\left(\vphantom{\rule{0cm}{0.9cm}}\right.$} ;
            \node at (6.8, 1.5) {$\left.\vphantom{\rule{0cm}{0.9cm}}\right)$} ;
            \node at (4.6,1.5) {\Large $\boldsymbol\times$} ;
            \node at (4.2, 2.2) {\Large$\boldsymbol\top$};
            \vector(5.3:0.7:5:6:0.25:0)
            \node at (1.05,0.4) {\large query feature};
            \node at (3.45,0.4) {\large key feature};
            \node at (5.95,0.4) {\large value};
            \node at (2.25, -0.1) {\large (position encoded)};
        \end{scope}
        \begin{scope}[shift={(0,0.2)}]
            \node[draw=blue,thick,fill=blue!20,inner sep=1.5mm] at (2.8,4.3) {\large Position-aware Permutation $\mathbf{P}_\pi^i$};
            \draw[-Latex,line width=1.2pt,black!70!white] (1.3,5.4) -- (1.3,4.7);
            \draw[-Latex,line width=1.2pt,black!70!white] (3.8,5.4) -- (3.8,4.7);
            \draw[-Latex,line width=1.2pt,black!70!white] (1.3,3.95) -- (1.3,3.1);
            \draw[-Latex,line width=1.2pt,black!70!white] (3.8,3.95) -- (3.8,3.1);
        \end{scope}
        \begin{scope}[shift={(0,5.2)}]
            \node[draw=green,thick,fill=green!20,inner sep=1.5mm] at (2.6,4.3) {\large \quad Feature Map $\boldsymbol\phi$\quad\ };
            \draw[-Latex,line width=1.2pt,black!70!white] (1.3,5.4) -- (1.3,4.7);
            \draw[-Latex,line width=1.2pt,black!70!white] (3.8,5.4) -- (3.8,4.7);
            \draw[-Latex,line width=1.2pt,black!70!white] (1.3,3.95) -- (1.3,3.1);
            \draw[-Latex,line width=1.2pt,black!70!white] (3.8,3.95) -- (3.8,3.1);
        \end{scope}
    \end{tikzpicture}
    \caption{Attention in Transformer, Performer and \name{}.
    Although attention is multi-headed in all of them, only one head is illustrated for clarity.
    \textbf{Transformer} applies softmax on dot-products of queries and keys to get the attention matrix, and then multiplies attention matrix and values to obtain outputs of attention module.
    \textbf{Performer} applies feature map, a non-linear projection, to queries and keys to get query features and key features. Then, it multiplies query features, key features and values from right to left.
    \textbf{\name{}} applies a position-aware permutation on query features and key features first, and then do multiplications the same way as Performer.
    Each token's query / key feature is illustrated as a row of blocks in the figure, and its elements are marked with different colors.
    The position-aware permutation permutes elements of each token's query / key feature along the \textit{head size} dimension in each attention head. Depending on the token's position, the permutation applied to query / key feature is different.
    Note that for Performer and \name{}, only the numerator in Equation~\ref{eq:new_pos_final} is illustrated, as the denominator is simpler than the numerator.
    }
    \label{fig:\name{}}
\end{figure}

\section{Related Work}

\paragraph{Efficient Transformers}
\label{sec:effcient_transformers}

Transformers suffer from complexity quadratic to the sequence length. Various methods have been proposed to improve the efficiency of Transformers. We classify them into three categories.
The first category of efficient Transformer omits the calculation of part of the attention matrix, exploiting the sparsity of the attention matrix. \citet{Kitaev:20} groups queries into buckets by local sensitive hash and computes intra-bucket attention weights only. \citet{Zaheer:20, Beltagy:20} limit attention matrix to specific sparse shapes. The second kind of efficient Transformers lowers matrix rank to reduce computation. \citet{Wang:20} projects keys and values to constant length independent of sequence lengths. \citet{Tay:20c} generates attention weights without keys. The third category of efficient Transformers, named Performer in this paper, leverages kernel methods to speed models up. \citet{Choromanski:20,Peng21} view attention weights as kernel function of queries and keys, so they can be approximated by random features. \citet{Katharopoulos:20} relaxes the approximation requirement and finds that the model still works. \citet{Likhosherstov:20, Kasai:21} implement the unidirectional Performer as RNN so that their memory footprint is constant.

\paragraph{Relative Position Encoding}

Transformer itself does not capture the positional information of tokens, as it is invariant to permutations of tokens. \citet{Vaswani:17} solves this problem by adding a position embedding vector to the input of Transformer. Because the added position embedding depends on the absolute positions of tokens in a sequence, it is called absolute position encoding. For better representation of positional relation between tokens, \citet{shaw-etal-2018-self} introduces relative position encoding to encode distances between tokens directly. There are two styles of relative position encoding. \citet{shaw-etal-2018-self} adds relative position embedding to keys and values, while \citet{dai-etal-2019-transformer} adds relative position embedding to queries and keys. \citet{Raffel20}, as the other style of relative position encoding, adds bias directly to the attention weights.

\paragraph{Concurrent Work}

\citet{Su21} introduces RoFormer with a new kind of relative position encoding named RoPE, which is interoperable with Performer. Briefly, RoPE is a multiplicative sinusoidal absolute position embedding that rotates query (feature) vectors and key (feature) vectors according to their positions.

However, to make RoPE independent of absolute position, they sacrifice the property of attention matrices that every row sums to one.
Moreover, they only discuss the possibility of integrating RoPE with Performer, but no experiment result is reported on such a model.

On the other hand, \name{}'s position encoding preserves the property of  attention matrices mentioned above.
In this paper, we compare the performance of \name{} with RoFormer through experiment. The result shows that \name{} fits the data better than RoFormer.

\section{Methods}

We propose an efficient relative position encoding that is compatible with Performer architecture. Performer with this new relative position encoding is named as \name{}, because it permutes elements of query feature and key feature to encode positional information. The difference among vanilla Transformer, Performer and \name{} is illustrated in Figure~\ref{fig:\name{}}.

In this section, we first introduce Transformer and Performer briefly, and then describe details of \name{}. 
For brevity and clarity, discussions in this section focus on a single head in multi-head attention. They can be directly applied to the whole multi-head attention.

\subsection{Transformer and Performer}
\label{sec:performer}

We give a brief introduction of Transformer and Performer's attention module in this section. Other parts of Transformer architecture \citep{Vaswani:17} are omitted as they are unmodified in Performer and \name{}.

The attention module in Transformer is a mapping from a sequence of vectors $\{\mathbf{x}^\mathrm{in}_i\}_{i=1}^{L}$ to another sequence of vectors $\{\mathbf{x}^\mathrm{out}_i\}_{i=1}^{L}$ with the same length $L$.
In the attention module, the input vectors are first linearly mapped to three representations, named \textbf{query}, \textbf{key} and \textbf{value}. Formally,
\begin{align}
\mathbf{q}_i=\mathbf{W}_q\mathbf{x}^\mathrm{in}_i,\  \mathbf{k}_i=\mathbf{W}_k\mathbf{x}^\mathrm{in}_i,\  \mathbf{v}_i=\mathbf{W}_v\mathbf{x}^\mathrm{in}_i,\ 
\end{align}
where $\mathbf{W}_q$, $\mathbf{W}_k$, $\mathbf{W}_v$ are transformation matrices for query, key and value, respectively.
Then, similarities between queries and keys are calculated. The similarities are normalized to produce attention weights
\begin{align}
\label{eq:attn_weight}
\alpha_{ij}=\frac{\mathrm{sim}(\mathbf{q}_i, \mathbf{k}_j)}{\sum_{l=1}^L\mathrm{sim}(\mathbf{q}_i, \mathbf{k}_l)},
\end{align}
where $\mathrm{sim}(\mathbf{q}_i, \mathbf{k}_j)$ is the similarity of vector $\mathbf{q}_i$ and vector $\mathbf{k}_j$.
Finally, output vectors $\mathbf{x}^\mathrm{out}_i$ are obtained by weighted sum of values with weight $\{\alpha_{ij}\}_{i,j=1}^L$.
\begin{align}
\mathbf{x}^\mathrm{out}_i=\sum_{j=1}^L\alpha_{ij}\mathbf{v}_j.
\end{align}
Vanilla Transformer \citep{Vaswani:17} adopts the following function as the similarity metric of queries and keys. 
\begin{align}
\mathrm{sim}_\mathrm{Trans}(\mathbf{q}_i, \mathbf{k}_j) = \exp{\left(\mathbf{q}_i^\top\mathbf{k}_j/\sqrt{d}\right)}.
\end{align}

To reduce computation and memory cost, Performer's similarity function is approximated with kernel trick.
\begin{align}
\label{eq:kernel}
\mathrm{sim}_\mathrm{Perf}(\mathbf{q}_i, \mathbf{k}_j)=\boldsymbol\phi(\mathbf{q}_i)^\top\boldsymbol\phi(\mathbf{k}_j),
\end{align}
where $\boldsymbol\phi(\cdot)$ is a non-linear \textbf{feature map} from $\mathbb{R}^d$ to $\mathbb{R}^m$ for some model-specific $m$,
so that the attention module can be expressed as follows.
\begin{align}
\mathbf{x}^\mathrm{out}_i=
\left(\frac{\boldsymbol\phi(\mathbf{q}_i)^\top\sum_{j=1}^L\boldsymbol\phi(\mathbf{k}_j)\mathbf{v}_j^\top}
{\boldsymbol\phi(\mathbf{q}_i)^\top\sum_{j=1}^L\boldsymbol\phi(\mathbf{k}_j)}
\right)^\top.
\end{align}
We call $\boldsymbol\phi(\mathbf{q}_i)$ as \textbf{query feature} and $\boldsymbol\phi(\mathbf{k}_i)$ as \textbf{key feature}.

In this way, the $O(L^2)$ attention weight matrix is not explicitly calculated, so that the attention module costs only $O(L)$ time and memory, rather than the $O(L^2)$ complexity as vanilla Transformer. Different Performers differ by the choice of the mapping $\boldsymbol\phi(\cdot)$. A simple working choice is the ReLU function $\boldsymbol\phi(\mathbf{x})=\max(\mathbf{x},\mathbf{0})$ \citep{Choromanski:20}.

\subsection{Relative Position Encoding for Performer}

In this section, we discuss adding relative position encoding to Performer. We choose to modify the similarity function (Equation~\ref{eq:kernel}) to encode positional information.
Specifically, we introduce an additional layer of position-dependent linear transformation over query features and key features. Now, the similary function becomes
\begin{align}
\label{eq:new_pos}
\mathrm{sim}_\mathrm{Perm}(\mathbf{q}_i,\mathbf{k}_j)=\big(\mathbf{M}_i\boldsymbol{\phi}(\mathbf{q}_i)\big)^\top\big(\mathbf{N}_j\boldsymbol{\phi}(\mathbf{k}_j)\big),
\end{align}
where $\mathbf{M}_i, \mathbf{N}_j\in\mathbb{R}^{m\times m}$ are matrices parameterized by token's position $i$ and $j$.

To ensure the similarity function depends only on the relative positions rather than absolute ones, $\mathbf{M}_i,\mathbf{N}_j$ must hold the following property.
\begin{property}[Relative]
\label{prop:relative}
$\mathbf{M}_i^\top\mathbf{N}_j$ is a function of $i-j$, i.e., it only depends on $i-j$.
\end{property}

To prevent the similarity function from exploding as the sequence length grows, we have
\begin{property}[Bounded]
\label{prop:bounded}
For a bidirectional model, there is an $B$ that for all $i,j\in\mathbb{Z}$, $\|\mathbf{M}_i^\top\mathbf{N}_j\|<B$.
For a unidirectional model, there is an $B$ that for all $i>j\in\mathbb{Z}$, $\|\mathbf{M}_i^\top\mathbf{N}_j\|<B$.
\end{property}

Additionally, the similarity function should be positive; otherwise, the model would be numerically unstable. If the similarity function alters between positive and negative values, in some cases, the denominator in Equation~\ref{eq:attn_weight} may be zero while its numerator is not zero, leading the output of attention module tend to infinity. To keep the similarity function positive, one simple but efficient solution is to make all elements of query features and key features positive \citep{Choromanski:20, Katharopoulos:20}.
\begin{property}[Positive]
\label{prop:positive}
The linear transformations corresponding to matrix $\mathbf{M}_i$ and $\mathbf{N}_j$ map $\mathbb{R}^m_+$ to $\mathbb{R}^m_+$.
\end{property}

We prove that, $\mathbf{M}_i^\top\mathbf{N}_j$ must be in a specific form to fulfill the requirement of Property~\ref{prop:relative}.
\begin{proposition}
\label{prop:ortho}
Let $\left\{\mathbf{M}_i\right\}_{i=-\infty}^\infty$ be a series of $l\times m$ matrices, $\left\{\mathbf{N}_i\right\}_{i=-\infty}^\infty$ be a series of $l\times n$ matrices.
Then,  $\mathbf{M}_i^\top\mathbf{N}_j$ only depends on $i-j$, if and only if that, there is an integer $l'$, matrices $\mathbf{R}\in\mathbb{R}^{l'\times m}$, $\mathbf{Q}\in\mathbb{R}^{l'\times n}$, and an invertible matrix $\mathbf{P}\in\mathbb{R}^{l'\times l'}$, such that
\begin{align}
    \mathbf{M}_i^\top\mathbf{N}_j=(\mathbf{P}^{-i\top}\mathbf{R})^\top(\mathbf{P}^j\mathbf{Q}),
\end{align}
\end{proposition}
Proof is given in Appendix.
Although this proposition does not impose any additional constraint on $\mathbf{M}_i$ and $\mathbf{N}_j$, it suggests that effectively we only need to consider the case that
\begin{align}
    \label{eq:invert}
    \mathbf{M}_i=\mathbf{P}^{-i\top}\mathbf{R}, \mathbf{N}_j=\mathbf{P}^j\mathbf{Q}
\end{align}

\subsection{\name{}}

Based on the analysis of the previous section, we introduce \name{} by selecting specific $\mathbf{P}, \mathbf{Q}, \mathbf{R}$ in Equation~\ref{eq:invert}.

To meet constraints imposed by Property~\ref{prop:bounded} and Property~\ref{prop:positive}, we choose the following solution for \name{}.
\begin{align}
\label{eq:permutation}
\mathbf{R}=\mathbf{Q}=\mathbf{I}, 
\mathbf{P}=r^{-1}\mathbf{P}_\pi,
\end{align}
where
$r=1$ for bidirectional models and $0<r<1$ for unidirectional models,
$\pi:\{1,2,\cdots,m\}\to\{1,2,\cdots,m\}$ is a permutation and $\mathbf{P}_\pi$ is the corresponding permutation matrix.
(A permutation matrix is a square binary matrix that has exactly one entry of 1 in each row and each column and 0s elsewhere. For permutation $\pi$ the corresponding permutation matrix $\mathbf{P}_\pi$ is the matrix that $\mathbf{P}_{\pi,ij}=1$ if $\pi(i)=j$; $\mathbf{P}_{\pi,ij}=0$ otherwise.)
Note that different attention heads may have different $\mathbf{P}_\pi$ and $r$, so that both long-term and short-term dependencies are captured.

Substitute Equation~\ref{eq:invert},~\ref{eq:permutation} into Equation~\ref{eq:new_pos}, we get the similarity function of \name{}
\begin{align}
\label{eq:new_pos_final}
\mathrm{sim}_\mathrm{Perm}(\mathbf{q}_i,\mathbf{k}_j)=\big(r^i\mathbf{P}_\pi^i\boldsymbol{\phi}(\mathbf{q}_i)\big)^\top\big(r^{-j}\mathbf{P}_\pi^j\boldsymbol{\phi}(\mathbf{k}_j)\big).
\end{align}

\name{} can encode relative positions up to the order of the permutation $\pi$. \citet{Goh:91} proves that the order of random permutation grows exponentially with the head size. For a model with the same size as BERT-base \citep{devlin-etal-2019-bert}, the dimension of queries / keys per attention head is 64, corresponding to an average order of over 3000. To further extend \name{}'s ability to encode long sequences, we choose different permutations for different attention heads, so that the longest distance \name{} can encode is the least common multiple of all permutations' orders, which can be up to 1e27 for a model with head size of 64.

There are two additional parameters \name{} introduces, $\pi$ and $r$. As $\pi$ is a discrete parameter that cannot be optimized by gradient-based methods, we treat it as a hyper-parameter of the model. We randomly sample $\pi$ at initialization of the neural network and fix its value during the whole training process. Although the model may get a better performance on training $\pi$, we find that a random permutation is good enough for \name{} to work, so we do not tune $\pi$ to save energy. Parameter $r$, on the other hand, can be optimized by gradient-based methods, but we also treat it as a hyper-parameter.

\subsection{Computational Cost}
\label{sec:speed}
We analyze computational cost of \name{} in this section. \name{} is as fast as Performer, which is the most efficient Transformer \citep{Tay:20} to our knowledge.

Let $L$ denote the length of the sequence, $H$ denote the number of heads in the model, and $m$ denote the per-head hidden dimension of query features and key features.

The computational overhead introduced by \name{} includes the computation of $\mathbf{P}_\pi^i$, the application of linear transformation $\mathbf{P}_\pi^i$ on query features and key features, as well as calculation of powers of $r$.

Multiplication of permutation matrices is equivalent to multiplication of corresponding permutations. In our case, it reads that
\begin{align}
\mathbf{P}_\pi^i=\mathbf{P}_{\pi^i},
\end{align}
where $\pi^i$ is the $i$-th power of permutation $\pi$ that
\begin{align}
\pi^i(x)=\pi(\pi^{i-1}(x))\mathrm{\ and\ }\pi^0(x)=x.
\end{align}
We can compute these $\pi^i$ and cache them before training and inference. This takes $O(LHm)$ time and $O(LHm)$ memory.

As $\mathbf{P}_\pi^i$ is a permutation matrix, there is no need to do cumbersome matrix-vector multiplication. Instead, a \texttt{gather} operation on query features and key features is enough. The memory and time complexity of this \texttt{gather} operation is equal to the size of query features and key features, i.e., $O(LHm)$.

Powers of scalar $r$ can be calculated easily.

Thus, the total overhead introduced by \name{} is $O(LHm)$. Since the complexity of attention in Performer is $O(LHm^2)$, this overhead is negligible.

\subsection{Trick for Two-Dimensional Case}
\label{sec:2d}
As Transformer-based models are getting popular in fields other than natural language processing these days, it is worth noting that \name{} is also applicable to 2D inputs like images and multi-modal documents \citep{Xu20}.

One naive way to deal with two-dimension inputs is to follow the convention in benchmark \citet{Tay:20}. Pixels in the 2D space are first flattened to an 1D sequence before fed into the model. However, this causes problems for relative position encoding. It makes the rightmost pixel in the first row adjacent to the leftmost pixel in the second row, so the relative position of these two distant pixels is extremely close in the 1D sequence, which is incorrect. It is almost impossible for the model to learn something meaningful out of the wrong relative position.

To remedy this, we adapt \name{}'s attention for 2D inputs. We permute some elements of the query / key feature according to a pixel's horizontal position, while others according to its vertical position. More precisely, we modified equation \ref{eq:new_pos_final} as follows
\begin{align}
\label{eq:2d}
\begin{split}
\mathrm{sim}_\mathrm{Perm}(\mathbf{q}_i,\mathbf{k}_j)=\big(\mathbf{P}^{x_i}_{\pi_x}\mathbf{P}^{y_i}_{\pi_y}\boldsymbol{\phi}(\mathbf{q}_i)\big)^\top\\
\big(\mathbf{P}^{x_j}_{\pi_x}\mathbf{P}^{y_j}_{\pi_y}\boldsymbol{\phi}(\mathbf{k}_j)\big),
\end{split}
\end{align}
where $(x_i, y_i)$, $(x_j, y_j)$ are coordinates of the $i$-th and $j$-th pixel, respectively. $\pi_x$ and $\pi_y$ are two permutations commutative with each other. 

\section{Experiments}

We evaluate bidirectional \name{} on Long-Range Arena, which consists of many long-sequence tasks. Unidirectional \name{} is evaluated on WikiText-103, a language modeling task.%
\footnote{
Long-Range Arena can be fetched from \url{https://github.com/google-research/long-range-arena}.
WikiText-103 can be fetched from \url{https://s3.amazonaws.com/research.metamind.io/wikitext/wikitext-103-v1.zip}.
}


\begin{table*}
    \centering
    \begin{tabular}{c||llll|l}
    \toprule
        Model & Text & Retrieval & Image & Pathfinder & Average \\
    \midrule
        Transformer & 63.99${}_{23}$ & 80.02${}_{26}$ & 42.83${}_{142}$ & 72.40${}_{165}$ & 64.81${}_{55}$ \\
        w/ sinusoidal pos. emb. & 64.06${}_{17}$ & 79.81${}_{33}$ & \textbf{43.30}${}_{152}$ & \textbf{73.00}${}_{183}$ & 65.04${}_{60}$ \\
    \midrule
        Reformer & 64.88 & 78.64 & 43.29 & 69.36 & 64.04 \\
        Linformer & 55.91 & 79.37 & 37.84 & 67.60 & 60.18 \\
        Nystr{\"{o}}mformer & 65.52 & 79.56 & 41.58 & 70.94 & 64.40 \\
        Performer & 63.95${}_{42}$ & 79.82${}_{30}$ & 43.08${}_{174}$ & 72.63${}_{108}$ & 64.87${}_{53}$ \\
        RoFormer & \textbf{66.00}${}_{17}$ & 75.27${}_{55}$ & 26.16${}_{109}$ & 58.87${}_{26}$ & 56.57${}_{31}$ \\
    \midrule
        \name{} & 65.95${}_{26}$ & \textbf{80.66}${}_{26}$ & 43.02${}_{52}$ & 72.91${}_{100}$ & \textbf{65.64}${}_{30}$ \\
        w/o 2D rel. pos. & 65.95${}_{26}$ & \textbf{80.66}${}_{26}$ & 36.10${}_{95}$ & 65.72${}_{59}$ & 62.10${}_{29}$ \\
        w/o Property~\ref{prop:positive} & 50.27 & 70.62 & 10.00 & 50.05 & 45.24 \\
    \bottomrule
    \end{tabular}
    \caption{Performance on Long-Range Arena in accuracy. Results of Transformer, Performer, RoFormer and all variants of \name{} are evaluated by us. Results for Reformer, Linformer and Nystr{\"{o}}mformer are taken from \citet{Xiong:21}. Numbers reported by us are average accuracies of five runs. Standard deviations are shown as subscripts, in units of 0.01.}
    \label{tab:main}
\end{table*}

\begin{table*}
    \centering
    \begin{tabular}{c||cccc}
    \toprule
        Model & Text (4K) & Retrieval (4K) & Image (1K) & Pathfinder (1K) \\
    \midrule
        Transformer & 622 (1.00$\times$) & 2404 (1.00$\times$) & 26 (1.00$\times$) & 180 (1.00$\times$) \\
        Reformer & 437 (0.70$\times$) & 1086 (0.45$\times$) & 30 (1.15$\times$) & 153 (0.85$\times$) \\
        Linformer & 323 (0.52$\times$) & 483 (0.20$\times$) & 14 (0.54$\times$) & 68 (0.38$\times$) \\
        Nystr{\"{o}}mformer & 332 (0.53$\times$) & 566 (0.24$\times$) & 15 (0.58$\times$) & 65 (0.36$\times$) \\
        Performer & 354 (0.57$\times$) & 553 (0.23$\times$) & 13 (0.50$\times$) & 61 (0.34$\times$) \\
        \name{} & 361 (0.58$\times$) & 550 (0.23$\times$) & 13 (0.50$\times$) & 62 (0.34$\times$) \\
    \midrule
        Performer + T5-style & 28585 & 81070 & 3697 & 24601 \\
        pos. emb. (estimated) & (46.0$\times$) & (33.7$\times$) & (142$\times$) & (137$\times$)\\
    \bottomrule
    \end{tabular}
    \caption{Training time for one epoch in seconds. Ratio to Transformer is included in parentheses. Lower is better. The sequence length of the first two tasks is 4000, while that of the last two is 1024.}
    \label{tab:time}
\end{table*}

\begin{table*}
    \centering
    \begin{tabular}{c||cccc}
    \toprule
        Model & Text (4K) & Retrieval (4K) & Image (1K) & Pathfinder (1K) \\
    \midrule
        Transformer & 72.03 (1.00$\times$) & 20.23 (1.00$\times$) & 3.28 (1.00$\times$) & 3.89 (1.00$\times$) \\
        Reformer & 30.60 (0.42$\times$) & 13.57 (0.67$\times$) & 13.44 (4.10$\times$) & 13.49 (3.47$\times$) \\
        Linformer & 19.82 (0.28$\times$) & 4.33 (0.21$\times$) & 3.31 (1.01$\times$) & 4.20 (1.08$\times$) \\
        Nystr{\"{o}}mformer & 23.55 (0.33$\times$) & 9.28 (0.46$\times$) & 7.01 (2.14$\times$) & 9.14 (2.35$\times$) \\
        Performer & 30.20 (0.42$\times$) & 5.09 (0.25$\times$) & 3.01 (0.92$\times$) & 3.80 (0.98$\times$) \\
        \name{} & 31.18 (0.43$\times$) & 5.18 (0.26$\times$) & 2.99 (0.91$\times$) & 3.89 (1.00$\times$) \\
    \bottomrule
    \end{tabular}
    \caption{Inference latency for one sample in milliseconds. Ratio to Transformer is included in parentheses. Lower is better. The sequence length of the first two tasks is 4000, while that of the last two is 1024.}
    \label{tab:infer_time}
\end{table*}

\subsection{Long-Range Arena}

Long-Range Arena \citep{Tay:20} is a benchmark for efficient Transformers.
It concentrates on efficient Transformers' performance on long sequences. The benchmark consists of five subtasks from various domains:
byte-level text classification, byte-level document retrieval, image classification on sequence of pixels, Pathfinder, and long ListOps.
We follow the evaluation protocol of \citet{Tay:20}, except that we exclude the long ListOps task from the benchmark, because a simple classifier on the first token\footnote{This classifier outputs \texttt{0} if the first token is \texttt{[MIN}, outputs \texttt{9} if the first token of the sequence is \texttt{[MAX}, and outputs \texttt{4} otherwise. It achieves an accuracy of 37.25 on the test set.} performs on par with the best model reported in \citet{Tay:20}.
In the four selected tasks, image classification has 10 labels, while the others are binary classification tasks.

\subsubsection{Setup and Implementations}
We compare our \name{} with the vanilla Transformer and Performer. A version of \citet{Su21} is also implemented on Performer for comparison. In addition, we also list performances of other efficient Transformers from \citet{Xiong:21}, including Reformer \citep{Kitaev:20}, Linformer \citep{Wang:20} and Nystr{\"{o}}mformer \citep{Xiong:21}. Conventional relative position encoding, such as \citet{shaw-etal-2018-self, dai-etal-2019-transformer, Raffel20}, is not included, as it is almost computational infeasible to apply them to such long sequences. 

For efficiency, we choose a simple feature map
\begin{align}
\label{eq:relu_kernel}
\boldsymbol\phi(\mathbf{x}) = \mathrm{max}(\mathbf{x},\mathbf{0})+\epsilon,
\end{align}
for both Performer and \name{}.
$\epsilon$ is added to the features to ensure that the denominator in Equation~\ref{eq:attn_weight} is not zero. We set $\epsilon=0.001$.

In this paper, all neural networks are trained from scratch. Learning rates are manually tuned on Transformer to match the results reported by other papers. Then, these hyper-parameters are fixed on training of Performer and \name{}. Model sizes are the same as those described in \citet{Tay:20}. The hidden dimension of query features and key features are four times of that of queries and keys. Absolute position embedding is disabled for \name{}. Models are optimized with Adam \citep{Kingma15}. More details of hyper-parameters can be found in the appendix. Each experiment is run five times and the average accuracy is reported. Experiments are done on machines with 8 V100 GPUs.

\subsubsection{Results}

\paragraph{Performance}

The results are summarized in Table \ref{tab:main}. It shows that the relative position encoding in \name{} significantly improves the performance of Performer in all the tasks, including both language tasks and vision tasks. It not only achieves better accuracy than existing efficient Transformers without relative position encoding, but also performs better than vanilla Transformer, as well as Performer with \citet{Su21}'s relative position encoding.

\paragraph{Efficiency}

We record the training time of each model on all the tasks, as well as their latency on inference. The result is listed in Table~\ref{tab:time} and Table~\ref{tab:infer_time}. It shows that Performer runs around two to three times faster than Transformer. The second line and the third line of the table indicate that \name{}'s speed is almost the same as that of Performer. This aligns with our analysis in \S~\ref{sec:speed} that the overhead of \name{} is negligible compared to the computation cost of Performer itself.

We take T5 \citep{Raffel20} as an example to illustrate that existing relative position encoding is computationally infeasible for long sequences. We train Performer with T5 with a few iterations to estimate the running time for one epoch. The result is shown in the last line of Table~\ref{tab:time}. It indicates that T5 is significantly slower than Transformer, not to say Performer.

\subsubsection{Ablation Study}

We evaluate whether 2D relative position encoding is useful for \name{}. We train \name{} with 1D relative position encoding, and the result is shown in the second last line of Table~\ref{tab:main}. As expected, its performance drops significantly for tasks with 2D inputs. Thus, 1D relative position encoding is harmful to vision tasks as discussed in \S~\ref{sec:2d}.

We also justify that Property~\ref{prop:positive} is necessary for \name{}, i.e., the transformation should preserve positiveness of query features and key features. We train a \name{} with the permutation matrix $\mathbf{P}_\pi$ replaced by a random orthogonal matrix. The result is listed in the last line of Table~\ref{tab:main}, that \name{} without Property~\ref{prop:positive} does not converge on most of the tasks.

\begin{table}
    \centering
    \begin{tabular}{c|c}
    \toprule
        Model & PPL \\
    \midrule
        Transformer\citep{Vaswani:17} & 30.18 \\
    \midrule
        Performer\citep{Choromanski:20} & 36.87 \\
        \name{} & 32.49 \\
        \name{} w/o $r$ & 35.76 \\
        \name {} w/o $\mathbf{P}_\pi$ & 33.08 \\
    \bottomrule
    \end{tabular}
    \caption{Perplexity (PPL) of models on test split of WikiText-103 language modeling dataset.}
    \label{tab:wikitext}
\end{table}

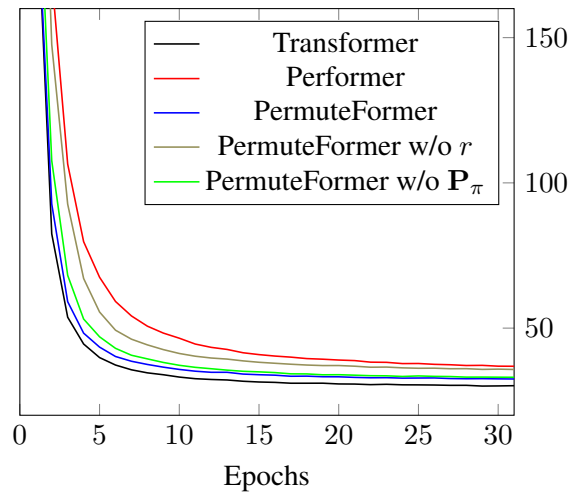
\begin{figure}
    \centering
    \begin{tikzpicture}
    \begin{axis}[
        title={Perplexity on WikiText-103},
        xlabel={Epochs},
        xmin=0, xmax=31,
        ymin=20, ymax=160,
        ytick pos=right,
        legend pos=north east,
        grid style=dashed,
        width=1.05*\linewidth,
        every axis plot/.append style={semithick},
    ]
    
    \addplot[
        color=black,
        ]
        coordinates {
            (1,207.49)(2,82.46)(3,53.75)(4,44.53)(5,39.86)
            (6,37.29)(7,35.65)(8,34.6)(9,33.95)(10,33.15)
            (11,32.59)(12,32.32)(13,32.15)(14,31.71)(15,31.45)
            (16,31.32)(17,31.03)(18,31.03)(19,31.04)(20,30.78)
            (21,30.74)(22,30.54)(23,30.66)(24,30.48)(25,30.46)
            (26,30.41)(27,30.3)(28,30.31)(29,30.05)(30,30.1)
            (31,30.18)
        };
        \addlegendentry{Transformer}
    
        \addplot[
        color=red,
        ]
        coordinates {
            (1,318.12)(2,171.05)(3,106.47)(4,79.76)(5,67.51)
            (6,59.17)(7,54.19)(8,50.67)(9,48.29)(10,46.53)
            (11,44.52)(12,43.38)(13,42.66)(14,41.52)(15,40.88)
            (16,40.38)(17,40.01)(18,39.51)(19,39.3)(20,39.01)
            (21,38.82)(22,38.28)(23,38.22)(24,37.79)(25,37.82)
            (26,37.53)(27,37.39)(28,37.13)(29,37.18)(30,36.9)
            (31,36.87)
        };
        \addlegendentry{Performer}
        
        \addplot[
        color=blue,
        ]
        coordinates {
            (1,207.83)(2,92.74)(3,59.1)(4,48.27)(5,43.46)
            (6,40.23)(7,38.59)(8,37.49)(9,36.56)(10,35.78)
            (11,35.17)(12,34.79)(13,34.79)(14,34.18)(15,33.96)
            (16,33.8)(17,33.43)(18,33.47)(19,33.26)(20,33.24)
            (21,33.04)(22,32.91)(23,32.94)(24,32.78)(25,32.82)
            (26,32.85)(27,32.64)(28,32.57)(29,32.6)(30,32.53)
            (31,32.49)
        };
        \addlegendentry{\name{}}
        
        \addplot[
        color=yellow!50!black,
        ]
        coordinates {
            (1,300.64)(2,147.64)(3,92.68)(4,67.09)(5,55.54)
            (6,49.28)(7,46.19)(8,44.23)(9,42.6)(10,41.27)
            (11,40.33)(12,39.71)(13,39.34)(14,38.72)(15,38.24)
            (16,37.9)(17,37.57)(18,37.29)(19,37.11)(20,37.12)
            (21,36.89)(22,36.54)(23,36.59)(24,36.32)(25,36.18)
            (26,36.21)(27,36.02)(28,36.04)(29,35.79)(30,35.88)
            (31,35.76)
        };
        \addlegendentry{\name{} w/o $r$}
        
        \addplot[
        color=green,
        ]
        coordinates {
            (1,229.12)(2,107.64)(3,68.17)(4,53.13)(5,47.01)
            (6,43.04)(7,40.66)(8,39.41)(9,38.18)(10,37.2)
            (11,36.5)(12,36.02)(13,35.53)(14,35.14)(15,34.94)
            (16,34.69)(17,34.26)(18,34.23)(19,33.93)(20,33.96)
            (21,33.8)(22,33.62)(23,33.58)(24,33.36)(25,33.54)
            (26,33.39)(27,33.33)(28,33.13)(29,33.14)(30,33.11)
            (31,33.08)
        };
        \addlegendentry{\name{} w/o $\mathbf{P}_\pi$}
        
    \end{axis}
    \end{tikzpicture}
    \caption{Trends of perplexity during training on test split of WikiText-103 language modeling dataset.}
    \label{fig:wikitext}
\end{figure}

\subsection{WikiText-103}

We evaluate unidirectional \name{} on WikiText-103 \citep{Merity17}. It is a language modeling dataset with about 103 million tokens extracted from verified articles on Wikipedia. 

\subsubsection{Setup and Implementations}
We compare \name{} with the vanilla Transformer and Performer. Models are implemented with \texttt{fairseq} \citep{ott-etal-2019-fairseq}. We adopt hyper-parameters suggested by \texttt{fairseq}%
\footnote{We use the same command-line options as described in \url{https://github.com/pytorch/fairseq/tree/master/examples/language_model}.}%
: 6 layers, hidden dimension of 512, feed forward dimension of 1024, 8 attention heads. Feature map is the same as Equation~\ref{eq:kernel}. $r$ takes its value in $[0.88, 0.99]$.
For comparison with absolute position encoding, we set the sequence length to 512.
Perplexity is measured on the test set. To avoid predicting tokens with little context at the beginning of a sequence, only the last 256 tokens are counted in the results.
Effects of $r$ and $\mathbf{P}_\pi$ are measured separately through ablation studies, i.e., removing $r$ or $\mathbf{P}_\pi$ in Equation~\ref{eq:new_pos_final}.

\subsubsection{Results}

The results for WikiText-103 are listed in Table~\ref{tab:wikitext}. We also plot trending of perplexity during training in Figure~\ref{fig:wikitext}. It shows that \name{} lowers the performance gap between Transformer and Performer. It also speeds up convergence of models.

The last two lines of Table~\ref{tab:wikitext} indicate that performance of \name{} drops without $r$ or $\mathbf{P}_\pi$. Thus, both $r$ and $\mathbf{P}_\pi$ are crucial for \name{}.
$r$ may be helpful for \name{} to focus on local context, while $\mathbf{P}_\pi$ is responsible for encoding relative positional information.

\section{Conclusions}

We discuss possible ways to add relative position encoding to Performer, a family of efficient Transformers scales linearly. Based on the analysis, we propose \name{}, a variant of Performer with position-aware permutation to encode relative positional information. While improving the performance, this novel relative position encoding introduces negligible overhead compared to the overall computational cost of Performer. Experiments show that it runs as fast as Performer.

Extensive experiments are conducted on \name{}, including byte-level text tasks and pixel-level image classification of Long-Range Arena, as well as language modeling on WikiText-103. Bidirectional \name{} is used for the former tasks, while unidirectional \name{} is adopted for the latter one. Results show that \name{} uniformly improves the performance of Performer, accelerates convergence, and achieves state-of-the-art on some tasks.

\section*{Ethical Considerations}

This paper does not introduce new datasets. All the experiments and discussions are based on public datasets, which have been widely used for years. This paper focuses on speeding up NLP models generally. It is not directly connected to specific real-world applications.

The purpose of this paper is to reduce the computational cost of Transformer without performance drop. We hope our work will reduce energy consumption for future work of NLP. We also try our best to reduce carbon cost in experiments, such as minimizing hyper-parameter tuning. It takes about 10 days on 8 V100 GPUs to get all the figures in this paper.

\bibliography{anthology,custom}

\begin{thebibliography}{29}
\expandafter\ifx\csname natexlab\endcsname\relax\def\natexlab#1{#1}\fi

\bibitem[{Baevski et~al.(2020)Baevski, Zhou, Mohamed, and Auli}]{Baevski:20}
Alexei Baevski, Yuhao Zhou, Abdelrahman Mohamed, and Michael Auli. 2020.
\newblock \href
  {https://proceedings.neurips.cc/paper/2020/hash/92d1e1eb1cd6f9fba3227870bb6d7f07-Abstract.html}
  {wav2vec 2.0: {A} framework for self-supervised learning of speech
  representations}.
\newblock In \emph{Advances in Neural Information Processing Systems 33: Annual
  Conference on Neural Information Processing Systems 2020, NeurIPS 2020,
  December 6-12, 2020, virtual}.

\bibitem[{Beltagy et~al.(2020)Beltagy, Peters, and Cohan}]{Beltagy:20}
Iz~Beltagy, Matthew~E. Peters, and Arman Cohan. 2020.
\newblock \href {http://arxiv.org/abs/2004.05150} {Longformer: The
  long-document transformer}.
\newblock \emph{CoRR}, abs/2004.05150.

\bibitem[{Choromanski et~al.(2020)Choromanski, Likhosherstov, Dohan, Song,
  Gane, Sarl{\'{o}}s, Hawkins, Davis, Mohiuddin, Kaiser, Belanger, Colwell, and
  Weller}]{Choromanski:20}
Krzysztof Choromanski, Valerii Likhosherstov, David Dohan, Xingyou Song,
  Andreea Gane, Tam{\'{a}}s Sarl{\'{o}}s, Peter Hawkins, Jared Davis, Afroz
  Mohiuddin, Lukasz Kaiser, David Belanger, Lucy Colwell, and Adrian Weller.
  2020.
\newblock \href {http://arxiv.org/abs/2009.14794} {Rethinking attention with
  performers}.
\newblock \emph{CoRR}, abs/2009.14794.

\bibitem[{Dai et~al.(2019)Dai, Yang, Yang, Carbonell, Le, and
  Salakhutdinov}]{dai-etal-2019-transformer}
Zihang Dai, Zhilin Yang, Yiming Yang, Jaime Carbonell, Quoc Le, and Ruslan
  Salakhutdinov. 2019.
\newblock \href {https://doi.org/10.18653/v1/P19-1285} {Transformer-{XL}:
  Attentive language models beyond a fixed-length context}.
\newblock In \emph{Proceedings of the 57th Annual Meeting of the Association
  for Computational Linguistics}, pages 2978--2988, Florence, Italy.
  Association for Computational Linguistics.

\bibitem[{Devlin et~al.(2019)Devlin, Chang, Lee, and
  Toutanova}]{devlin-etal-2019-bert}
Jacob Devlin, Ming-Wei Chang, Kenton Lee, and Kristina Toutanova. 2019.
\newblock \href {https://doi.org/10.18653/v1/N19-1423} {{BERT}: Pre-training of
  deep bidirectional transformers for language understanding}.
\newblock In \emph{Proceedings of the 2019 Conference of the North {A}merican
  Chapter of the Association for Computational Linguistics: Human Language
  Technologies, Volume 1 (Long and Short Papers)}, pages 4171--4186,
  Minneapolis, Minnesota. Association for Computational Linguistics.

\bibitem[{Dosovitskiy et~al.(2020)Dosovitskiy, Beyer, Kolesnikov, Weissenborn,
  Zhai, Unterthiner, Dehghani, Minderer, Heigold, Gelly, Uszkoreit, and
  Houlsby}]{Dosovitskiy:20}
Alexey Dosovitskiy, Lucas Beyer, Alexander Kolesnikov, Dirk Weissenborn,
  Xiaohua Zhai, Thomas Unterthiner, Mostafa Dehghani, Matthias Minderer, Georg
  Heigold, Sylvain Gelly, Jakob Uszkoreit, and Neil Houlsby. 2020.
\newblock \href {http://arxiv.org/abs/2010.11929} {An image is worth 16x16
  words: Transformers for image recognition at scale}.
\newblock \emph{CoRR}, abs/2010.11929.

\bibitem[{Goh and Schmutz(1991)}]{Goh:91}
William M.~Y. Goh and Eric Schmutz. 1991.
\newblock \href {https://doi.org/https://doi.org/10.1112/blms/23.1.34} {The
  expected order of a random permutation}.
\newblock \emph{Bulletin of the London Mathematical Society}, 23(1):34--42.

\bibitem[{He et~al.(2020)He, Liu, Gao, and Chen}]{He20}
Pengcheng He, Xiaodong Liu, Jianfeng Gao, and Weizhu Chen. 2020.
\newblock \href {http://arxiv.org/abs/2006.03654} {Deberta: Decoding-enhanced
  {BERT} with disentangled attention}.
\newblock \emph{CoRR}, abs/2006.03654.

\bibitem[{Kasai et~al.(2021)Kasai, Peng, Zhang, Yogatama, Ilharco, Pappas, Mao,
  Chen, and Smith}]{Kasai:21}
Jungo Kasai, Hao Peng, Yizhe Zhang, Dani Yogatama, Gabriel Ilharco, Nikolaos
  Pappas, Yi~Mao, Weizhu Chen, and Noah~A. Smith. 2021.
\newblock \href {http://arxiv.org/abs/2103.13076} {Finetuning pretrained
  transformers into rnns}.
\newblock \emph{CoRR}, abs/2103.13076.

\bibitem[{Katharopoulos et~al.(2020)Katharopoulos, Vyas, Pappas, and
  Fleuret}]{Katharopoulos:20}
Angelos Katharopoulos, Apoorv Vyas, Nikolaos Pappas, and Fran{\c{c}}ois
  Fleuret. 2020.
\newblock \href {http://arxiv.org/abs/2006.16236} {Transformers are rnns: Fast
  autoregressive transformers with linear attention}.
\newblock \emph{CoRR}, abs/2006.16236.

\bibitem[{Kingma and Ba(2015)}]{Kingma15}
Diederik~P. Kingma and Jimmy Ba. 2015.
\newblock \href {http://arxiv.org/abs/1412.6980} {Adam: {A} method for
  stochastic optimization}.
\newblock In \emph{3rd International Conference on Learning Representations,
  {ICLR} 2015, San Diego, CA, USA, May 7-9, 2015, Conference Track
  Proceedings}.

\bibitem[{Kitaev et~al.(2020)Kitaev, Kaiser, and Levskaya}]{Kitaev:20}
Nikita Kitaev, Lukasz Kaiser, and Anselm Levskaya. 2020.
\newblock \href {https://openreview.net/forum?id=rkgNKkHtvB} {Reformer: The
  efficient transformer}.
\newblock In \emph{8th International Conference on Learning Representations,
  {ICLR} 2020, Addis Ababa, Ethiopia, April 26-30, 2020}. OpenReview.net.

\bibitem[{Likhosherstov et~al.(2020)Likhosherstov, Choromanski, Davis, Song,
  and Weller}]{Likhosherstov:20}
Valerii Likhosherstov, Krzysztof Choromanski, Jared Davis, Xingyou Song, and
  Adrian Weller. 2020.
\newblock \href {http://arxiv.org/abs/2012.11346} {Sub-linear memory: How to
  make performers slim}.
\newblock \emph{CoRR}, abs/2012.11346.

\bibitem[{Merity et~al.(2017)Merity, Xiong, Bradbury, and Socher}]{Merity17}
Stephen Merity, Caiming Xiong, James Bradbury, and Richard Socher. 2017.
\newblock \href {https://openreview.net/forum?id=Byj72udxe} {Pointer sentinel
  mixture models}.
\newblock In \emph{5th International Conference on Learning Representations,
  {ICLR} 2017, Toulon, France, April 24-26, 2017, Conference Track
  Proceedings}. OpenReview.net.

\bibitem[{Ott et~al.(2019)Ott, Edunov, Baevski, Fan, Gross, Ng, Grangier, and
  Auli}]{ott-etal-2019-fairseq}
Myle Ott, Sergey Edunov, Alexei Baevski, Angela Fan, Sam Gross, Nathan Ng,
  David Grangier, and Michael Auli. 2019.
\newblock \href {https://doi.org/10.18653/v1/N19-4009} {fairseq: A fast,
  extensible toolkit for sequence modeling}.
\newblock In \emph{Proceedings of the 2019 Conference of the North {A}merican
  Chapter of the Association for Computational Linguistics (Demonstrations)},
  pages 48--53, Minneapolis, Minnesota. Association for Computational
  Linguistics.

\bibitem[{Peng et~al.(2021)Peng, Pappas, Yogatama, Schwartz, Smith, and
  Kong}]{Peng21}
Hao Peng, Nikolaos Pappas, Dani Yogatama, Roy Schwartz, Noah~A. Smith, and
  Lingpeng Kong. 2021.
\newblock \href {http://arxiv.org/abs/2103.02143} {Random feature attention}.
\newblock \emph{CoRR}, abs/2103.02143.

\bibitem[{Raffel et~al.(2020)Raffel, Shazeer, Roberts, Lee, Narang, Matena,
  Zhou, Li, and Liu}]{Raffel20}
Colin Raffel, Noam Shazeer, Adam Roberts, Katherine Lee, Sharan Narang, Michael
  Matena, Yanqi Zhou, Wei Li, and Peter~J. Liu. 2020.
\newblock \href {http://jmlr.org/papers/v21/20-074.html} {Exploring the limits
  of transfer learning with a unified text-to-text transformer}.
\newblock \emph{J. Mach. Learn. Res.}, 21:140:1--140:67.

\bibitem[{Shaw et~al.(2018)Shaw, Uszkoreit, and Vaswani}]{shaw-etal-2018-self}
Peter Shaw, Jakob Uszkoreit, and Ashish Vaswani. 2018.
\newblock \href {https://doi.org/10.18653/v1/N18-2074} {Self-attention with
  relative position representations}.
\newblock In \emph{Proceedings of the 2018 Conference of the North {A}merican
  Chapter of the Association for Computational Linguistics: Human Language
  Technologies, Volume 2 (Short Papers)}, pages 464--468, New Orleans,
  Louisiana. Association for Computational Linguistics.

\bibitem[{Su et~al.(2021)Su, Lu, Pan, Wen, and Liu}]{Su21}
Jianlin Su, Yu~Lu, Shengfeng Pan, Bo~Wen, and Yunfeng Liu. 2021.
\newblock \href {http://arxiv.org/abs/2104.09864} {Roformer: Enhanced
  transformer with rotary position embedding}.
\newblock \emph{CoRR}, abs/2104.09864.

\bibitem[{Tan and Bansal(2019)}]{tan-bansal-2019-lxmert}
Hao Tan and Mohit Bansal. 2019.
\newblock \href {https://doi.org/10.18653/v1/D19-1514} {{LXMERT}: Learning
  cross-modality encoder representations from transformers}.
\newblock In \emph{Proceedings of the 2019 Conference on Empirical Methods in
  Natural Language Processing and the 9th International Joint Conference on
  Natural Language Processing (EMNLP-IJCNLP)}, pages 5100--5111, Hong Kong,
  China. Association for Computational Linguistics.

\bibitem[{Tay et~al.(2020{\natexlab{a}})Tay, Bahri, Metzler, Juan, Zhao, and
  Zheng}]{Tay:20c}
Yi~Tay, Dara Bahri, Donald Metzler, Da{-}Cheng Juan, Zhe Zhao, and Che Zheng.
  2020{\natexlab{a}}.
\newblock \href {http://arxiv.org/abs/2005.00743} {Synthesizer: Rethinking
  self-attention in transformer models}.
\newblock \emph{CoRR}, abs/2005.00743.

\bibitem[{Tay et~al.(2020{\natexlab{b}})Tay, Dehghani, Abnar, Shen, Bahri,
  Pham, Rao, Yang, Ruder, and Metzler}]{Tay:20}
Yi~Tay, Mostafa Dehghani, Samira Abnar, Yikang Shen, Dara Bahri, Philip Pham,
  Jinfeng Rao, Liu Yang, Sebastian Ruder, and Donald Metzler.
  2020{\natexlab{b}}.
\newblock \href {http://arxiv.org/abs/2011.04006} {Long range arena: {A}
  benchmark for efficient transformers}.
\newblock \emph{CoRR}, abs/2011.04006.

\bibitem[{Tay et~al.(2020{\natexlab{c}})Tay, Dehghani, Bahri, and
  Metzler}]{Tay:20b}
Yi~Tay, Mostafa Dehghani, Dara Bahri, and Donald Metzler. 2020{\natexlab{c}}.
\newblock \href {http://arxiv.org/abs/2009.06732} {Efficient transformers: {A}
  survey}.
\newblock \emph{CoRR}, abs/2009.06732.

\bibitem[{Vaswani et~al.(2017)Vaswani, Shazeer, Parmar, Uszkoreit, Jones,
  Gomez, Kaiser, and Polosukhin}]{Vaswani:17}
Ashish Vaswani, Noam Shazeer, Niki Parmar, Jakob Uszkoreit, Llion Jones,
  Aidan~N. Gomez, Lukasz Kaiser, and Illia Polosukhin. 2017.
\newblock \href
  {https://proceedings.neurips.cc/paper/2017/hash/3f5ee243547dee91fbd053c1c4a845aa-Abstract.html}
  {Attention is all you need}.
\newblock In \emph{Advances in Neural Information Processing Systems 30: Annual
  Conference on Neural Information Processing Systems 2017, December 4-9, 2017,
  Long Beach, CA, {USA}}, pages 5998--6008.

\bibitem[{Wang et~al.(2020)Wang, Li, Khabsa, Fang, and Ma}]{Wang:20}
Sinong Wang, Belinda~Z. Li, Madian Khabsa, Han Fang, and Hao Ma. 2020.
\newblock \href {http://arxiv.org/abs/2006.04768} {Linformer: Self-attention
  with linear complexity}.
\newblock \emph{CoRR}, abs/2006.04768.

\bibitem[{Xiong et~al.(2021)Xiong, Zeng, Chakraborty, Tan, Fung, Li, and
  Singh}]{Xiong:21}
Yunyang Xiong, Zhanpeng Zeng, Rudrasis Chakraborty, Mingxing Tan, Glenn Fung,
  Yin Li, and Vikas Singh. 2021.
\newblock \href {http://arxiv.org/abs/2102.03902} {Nystr{\"{o}}mformer: {A}
  nystr{\"{o}}m-based algorithm for approximating self-attention}.
\newblock \emph{CoRR}, abs/2102.03902.

\bibitem[{Xu et~al.(2020)Xu, Li, Cui, Huang, Wei, and Zhou}]{Xu20}
Yiheng Xu, Minghao Li, Lei Cui, Shaohan Huang, Furu Wei, and Ming Zhou. 2020.
\newblock \href {https://doi.org/10.1145/3394486.3403172} {Layoutlm:
  Pre-training of text and layout for document image understanding}.
\newblock In \emph{{KDD} '20: The 26th {ACM} {SIGKDD} Conference on Knowledge
  Discovery and Data Mining, Virtual Event, CA, USA, August 23-27, 2020}, pages
  1192--1200. {ACM}.

\bibitem[{Yang et~al.(2019)Yang, Dai, Yang, Carbonell, Salakhutdinov, and
  Le}]{Yang19}
Zhilin Yang, Zihang Dai, Yiming Yang, Jaime~G. Carbonell, Ruslan Salakhutdinov,
  and Quoc~V. Le. 2019.
\newblock \href
  {https://proceedings.neurips.cc/paper/2019/hash/dc6a7e655d7e5840e66733e9ee67cc69-Abstract.html}
  {Xlnet: Generalized autoregressive pretraining for language understanding}.
\newblock In \emph{Advances in Neural Information Processing Systems 32: Annual
  Conference on Neural Information Processing Systems 2019, NeurIPS 2019,
  December 8-14, 2019, Vancouver, BC, Canada}, pages 5754--5764.

\bibitem[{Zaheer et~al.(2020)Zaheer, Guruganesh, Dubey, Ainslie, Alberti,
  Onta{\~{n}}{\'{o}}n, Pham, Ravula, Wang, Yang, and Ahmed}]{Zaheer:20}
Manzil Zaheer, Guru Guruganesh, Kumar~Avinava Dubey, Joshua Ainslie, Chris
  Alberti, Santiago Onta{\~{n}}{\'{o}}n, Philip Pham, Anirudh Ravula, Qifan
  Wang, Li~Yang, and Amr Ahmed. 2020.
\newblock \href
  {https://proceedings.neurips.cc/paper/2020/hash/c8512d142a2d849725f31a9a7a361ab9-Abstract.html}
  {Big bird: Transformers for longer sequences}.
\newblock In \emph{Advances in Neural Information Processing Systems 33: Annual
  Conference on Neural Information Processing Systems 2020, NeurIPS 2020,
  December 6-12, 2020, virtual}.

\end{thebibliography}
\bibliographystyle{acl_natbib}

\appendix

\section{Proof of Proposition 1}
\label{sec:appendix_ortho}

\begin{lemma}
\label{lemma1}
Let
\begin{align}
    \left\{\mathbf{M}_i\right\}_{i=-\infty}^\infty\in\mathbb{R}^{l\times m}, \\
    \left\{\mathbf{N}_i\right\}_{i=-\infty}^\infty\in\mathbb{R}^{l\times n}.
\end{align}
Assume
\begin{align}
    \forall i,j,k\in\mathbb{Z}, \mathbf{M}_i^\top\mathbf{N}_j=\mathbf{M}_{i+k}^\top\mathbf{N}_{j+k}.
\end{align}
Then, there exists
\begin{align}
    \left\{\mathbf{M}'_i\right\}_{i=-\infty}^\infty\in\mathbb{R}^{l'\times m'}, \\
    \left\{\mathbf{N}'_i\right\}_{i=-\infty}^\infty\in\mathbb{R}^{l'\times n'}, \\
    \mathbf{P}\in\mathbb{R}^{m'\times m}, \mathbf{Q}\in\mathbb{R}^{n'\times n},
\end{align}
such that
\begin{align}
    \label{eq:lemma1_res_invariance}
    \forall i,j,k\in\mathbb{Z}, {\mathbf{M}'}_i^\top\mathbf{N}'_j={\mathbf{M}'}_{i+k}^\top\mathbf{N}'_{j+k}, \\
    \label{eq:lemma1_res_equivalence}
    \forall i,j\in\mathbb{Z}, \mathbf{M}_i^\top\mathbf{N}_j=(\mathbf{M}'_i\mathbf{P})^\top(\mathbf{N}'_j\mathbf{Q}), \\
    \label{eq:lemma1_res_im}
    \sum_{i=-\infty}^{\infty}\mathrm{im}(\mathbf{M}'_i)=\mathbb{R}^{l'}, \sum_{i=-\infty}^{\infty}\mathrm{im}(\mathbf{N}'_i)=\mathbb{R}^{l'}, \\
    \label{eq:lemma1_res_ker}
    \forall i\in\mathbb{Z}, \ker(\mathbf{M}'_i)=\{\mathbf{0}\}, \ker(\mathbf{N}'_i)=\{\mathbf{0}\}.
\end{align}
\end{lemma}

\begin{proof}
Induction on $l,m,n$.

If $l=m=n=0$, then $\mathbf{M}'_i=\mathbf{M}_i$, $\mathbf{N}'_i=\mathbf{N}_i$, $\mathbf{P}=\mathbf{I}_m$, $\mathbf{Q}=\mathbf{I}_n$ satisfies Equation~\ref{eq:lemma1_res_invariance}-\ref{eq:lemma1_res_ker}.

Obviously, $\mathbf{M}'_i=\mathbf{M}_i$, $\mathbf{N}'_i=\mathbf{N}_i$, $\mathbf{P}=\mathbf{I}_m$, $\mathbf{Q}=\mathbf{I}_n$ satisfies Equation~\ref{eq:lemma1_res_invariance}-\ref{eq:lemma1_res_equivalence}.

\textit{Case 1)} It does not satisfy Equation~\ref{eq:lemma1_res_im}. Without loss of generality, assume $\sum_{i=-\infty}^{\infty}\mathrm{im}(\mathbf{M}_i)\neq\mathbb{R}^{l}$. Then $(\sum_{i=-\infty}^{\infty}\mathrm{im}(\mathbf{M}_i))^\perp\neq\{\mathbf{0}\}$.
Let unit vector $\mathbf{x}\in(\sum_{i=-\infty}^{\infty}\mathrm{im}(\mathbf{M}_i))^\perp$. For any $i$, $\mathbf{x}\in\mathrm{im}(\mathbf{M}_i)^\perp$, so $\mathbf{M}_i=(\mathbf{I}_l-\mathbf{xx}^\top)\mathbf{M}_i$. Since $\mathrm{rank}(\mathbf{I}_l-\mathbf{xx}^\top)=l-1$, there is $\mathbf{A}\in\mathbb{R}^{(l-1)\times l}$ such that $\mathbf{I}_l-\mathbf{xx}^\top=\mathbf{A}^\top\mathbf{A}$.

Let $\Tilde{\mathbf{M}}_i=\mathbf{A}\mathbf{M}_i$, $\Tilde{\mathbf{N}}_i=\mathbf{A}\mathbf{N}_i$. Then,
$\mathbf{M}_i^\top\mathbf{N}_j=\Tilde{\mathbf{M}}_i^\top\Tilde{\mathbf{N}}_j$.
By induction, there is $\mathbf{M}'_i$, $\mathbf{N}'_i$, $\mathbf{P}$, $\mathbf{Q}$ that satisfies Equation~\ref{eq:lemma1_res_invariance}-\ref{eq:lemma1_res_ker}.

\textit{Case 2)} It satisfies Equation~\ref{eq:lemma1_res_im}, but not Equation~\ref{eq:lemma1_res_ker}. Without loss of generality, assume unit vector $\mathbf{x}\in\ker(\mathbf{N}_i)$ for some $i$. Then, for any $j,k$, $\mathbf{M}_k^\top\mathbf{N}_j\mathbf{x}=\mathbf{M}_{k+i-j}^\top\mathbf{N}_i\mathbf{x}=\mathbf{0}$. Thus, $\mathbf{N}_j\mathbf{x}\in\mathrm{im}(\mathbf{M}_k)^\perp$ for any $k$. Equivalently, $\mathbf{N}_j\mathbf{x}\in(\sum_{k=-\infty}^{\infty}\mathrm{im}(\mathbf{M}_k))^\perp$. By Equation~\ref{eq:lemma1_res_im}, $\mathbf{N}_j\mathbf{x}=\mathbf{0}$. So $\mathbf{x}\in\ker(\mathbf{N}_j)$ for any $j$.

Therefore, for any $j\in\mathbb{Z}$, $\mathbf{N}_j=\mathbf{N}_j(\mathbf{I}_n-\mathbf{xx}^\top)$. Since $\mathrm{rank}(\mathbf{I}_n-\mathbf{xx}^\top)=n-1$, there is $\mathbf{A}\in\mathbb{R}^{(n-1)\times n}$ such that $\mathbf{I}_n-\mathbf{xx}^\top=\mathbf{A}^\top\mathbf{A}$.

Let $\Tilde{\mathbf{M}}_i=\mathbf{M}_i$, $\Tilde{\mathbf{N}}_i=\mathbf{N}_i\mathbf{A}^\top$.
Then $\Tilde{\mathbf{M}}_i^\top\Tilde{\mathbf{N}}_j=\mathbf{M}_i^\top\mathbf{N}_j\mathbf{A}^\top=\mathbf{M}_{i+k}^\top\mathbf{N}_{j+k}\mathbf{A}^\top=\Tilde{\mathbf{M}}_{i+k}^\top\Tilde{\mathbf{N}}_{j+k}$. By induction we have $\Tilde{\mathbf{M}}'_i$, $\Tilde{\mathbf{N}}'_i$, $\Tilde{\mathbf{P}}$, $\Tilde{\mathbf{Q}}$ that
\begin{align*}
    \forall i,j,k\in\mathbb{Z}, {\Tilde{\mathbf{M}}'}_i{}^\top\Tilde{\mathbf{N}}'_j={\Tilde{\mathbf{M}}'}_{i+k}{}^\top\Tilde{\mathbf{N}}'_{j+k}, \\
    \forall i,j\in\mathbb{Z}, \Tilde{\mathbf{M}}_i^\top\Tilde{\mathbf{N}}_j=(\Tilde{\mathbf{M}}'_i\Tilde{\mathbf{P}})^\top(\Tilde{\mathbf{N}}'_j\Tilde{\mathbf{Q}}), \\
    \sum_{i=-\infty}^{\infty}\mathrm{im}(\Tilde{\mathbf{M}}'_i)=\mathbb{R}^{\Tilde{l}'}, \sum_{i=-\infty}^{\infty}\mathrm{im}(\Tilde{\mathbf{N}}'_i)=\mathbb{R}^{\Tilde{l}'}, \\
    \forall i\in\mathbb{Z}, \ker(\Tilde{\mathbf{M}}'_i)=\{\mathbf{0}\}, \ker(\Tilde{\mathbf{N}}'_i)=\{\mathbf{0}\}.
\end{align*}
So $\mathbf{M}'_i=\Tilde{\mathbf{M}}'_i$, $\mathbf{N}'_i=\Tilde{\mathbf{N}}'_i$, $\mathbf{P}=\Tilde{\mathbf{P}}$, $\mathbf{Q}=\Tilde{\mathbf{Q}}\mathbf{A}$ satisfies Equation~\ref{eq:lemma1_res_invariance}-\ref{eq:lemma1_res_ker}.

\textit{Case 3)} It satisfies both Equation~\ref{eq:lemma1_res_im} and Equation~\ref{eq:lemma1_res_ker}. Nothing to prove.
\end{proof}

\begin{lemma}
\label{lemma2}
Let
\begin{align}
    \left\{\mathbf{M}_i\right\}_{i=-\infty}^\infty\in\mathbb{R}^{l\times m}, \\
    \left\{\mathbf{N}_i\right\}_{i=-\infty}^\infty\in\mathbb{R}^{l\times n}.
\end{align}
Assume
\begin{align}
    \forall i,j,k\in\mathbb{Z}, \mathbf{M}_i^\top\mathbf{N}_j=\mathbf{M}_{i+k}^\top\mathbf{N}_{j+k}, \\
    \sum_{i=-\infty}^{\infty}\mathrm{im}(\mathbf{M}_i)=\mathbb{R}^{l}, \sum_{i=-\infty}^{\infty}\mathrm{im}(\mathbf{N}_i)=\mathbb{R}^{l}, \\
    \forall i\in\mathbb{Z}, \ker(\mathbf{M}_i)=\{\mathbf{0}\}, \ker(\mathbf{N}_i)=\{\mathbf{0}\}.
\end{align}
Then, there exists
\begin{align}
    \left\{\mathbf{M}'_i\right\}_{i=-\infty}^\infty\in\mathbb{R}^{l\times l}, \\
    \left\{\mathbf{N}'_i\right\}_{i=-\infty}^\infty\in\mathbb{R}^{l\times l}, \\
    \mathbf{P}\in\mathbb{R}^{l\times m}, \mathbf{Q}\in\mathbb{R}^{l\times n},
\end{align}
such that
\begin{align}
    \label{eq:lemma2_res_invariance}
    \forall i,j,k\in\mathbb{Z}, {\mathbf{M}'}_i^\top\mathbf{N}'_j={\mathbf{M}'}_{i+k}^\top\mathbf{N}'_{j+k}, \\
    \label{eq:lemma2_res_equivalence}
    \forall i,j\in\mathbb{Z}, \mathbf{M}_i^\top\mathbf{N}_j=(\mathbf{M}'_i\mathbf{P})^\top(\mathbf{N}'_j\mathbf{Q}), \\
    \label{eq:lemma2_res_im}
    \sum_{i=-\infty}^{\infty}\mathrm{im}(\mathbf{M}'_i)=\mathbb{R}^{l}, \sum_{i=-\infty}^{\infty}\mathrm{im}(\mathbf{N}'_i)=\mathbb{R}^{l}, \\
    \label{eq:lemma2_res_ker}
    \forall i\in\mathbb{Z}, \ker(\mathbf{M}'_i)=\{\mathbf{0}\}, \ker(\mathbf{N}'_i)=\{\mathbf{0}\}.
\end{align}
\end{lemma}

\begin{proof}
Induction on $l-m$, $l-n$.

If $l-m=l-n=0$, then $\mathbf{M}'_i=\mathbf{M}_i$, $\mathbf{N}'_i=\mathbf{N}_i$, $\mathbf{P}=\mathbf{I}_l$, $\mathbf{Q}=\mathbf{I}_l$ satisfies Equation~\ref{eq:lemma2_res_invariance}-\ref{eq:lemma2_res_ker}.

Without loss of generality, we only need to discuss the case that $n<l$.

If $n<l$, $\mathrm{im}(\mathbf{N}_0)\neq\mathbb{R}^l$.  On the other hand, $\sum_{i=-\infty}^{\infty}\mathrm{im}(\mathbf{N}_i)=\mathbb{R}^{l}$. So there is a column of $\mathbf{N}_p$ for some $p\neq0$ that in $\mathbb{R}^{l}\backslash\mathrm{im}(\mathbf{N}_0)$. More generally, there is a vector $\mathbf{e}\in\mathbb{R}^n$ and an integer $p$, that $\mathbf{N}_p\mathbf{e}\in\mathbb{R}^{l}\backslash\mathrm{im}(\mathbf{N}_0)$.

Let $\Tilde{\mathbf{M}}_i=\mathbf{M}_i$, $\Tilde{\mathbf{N}}_i=[\mathbf{N}_i, \mathbf{N}_{p+i}\mathbf{e}]$, $\mathbf{A}=[\mathbf{I}_n, \mathbf{0}_n]^\top$. Then,
\begin{align}
    &\Tilde{\mathbf{M}}_i^\top\Tilde{\mathbf{N}}_j\mathbf{A}=\mathbf{M}_i^\top\mathbf{N}_j. \\
    \begin{split}
    &\Tilde{\mathbf{M}}_i^\top\Tilde{\mathbf{N}}_j \\
    =&[\mathbf{M}_i^\top\mathbf{N}_j,\mathbf{M}_i^\top\mathbf{N}_{p+j}\mathbf{e}] \\
    =&[\mathbf{M}_{i+k}^\top\mathbf{N}_{j+k},\mathbf{M}_{i+k}^\top\mathbf{N}_{p+j+k}\mathbf{e}] \\
    =&\Tilde{\mathbf{M}}_{i+k}^\top\Tilde{\mathbf{N}}_{j+k}.
    \end{split} \\
    &\sum_{i=-\infty}^{\infty}\mathrm{im}(\Tilde{\mathbf{M}}_i)=\sum_{i=-\infty}^{\infty}\mathrm{im}(\mathbf{M}_i)=\mathbb{R}^{l}. \\
    &\forall i\in\mathbb{Z}, \ker(\Tilde{\mathbf{M}}_i)=\ker(\mathbf{M}_i)=\{\mathbf{0}\}. \\
    &\mathbb{R}^{l}\supset\sum_{i=-\infty}^{\infty}\mathrm{im}(\Tilde{\mathbf{N}}_i)\supset\sum_{i=-\infty}^{\infty}\mathrm{im}(\mathbf{N}_i)=\mathbb{R}^{l}.
\end{align}
If for some $i$, $\ker(\Tilde{\mathbf{N}}_i)\neq\{\mathbf{0}\}$, let $\mathbf{x}$ be a non-zero vector in $\ker(\Tilde{\mathbf{N}}_i)$. Then, for any $k$, $\Tilde{\mathbf{M}}_k^\top\Tilde{\mathbf{N}}_0\mathbf{x}=\Tilde{\mathbf{M}}_{k+i}^\top\Tilde{\mathbf{N}}_i\mathbf{x}=\mathbf{0}$.
Thus, $\Tilde{\mathbf{N}}_0\mathbf{x}\in\mathrm{im}(\Tilde{\mathbf{M}}_k)^\perp$ for any $k$. Equivalently, $\Tilde{\mathbf{N}}_0\mathbf{x}\in(\sum_{k=-\infty}^{\infty}\mathrm{im}(\Tilde{\mathbf{M}}_k))^\perp=\{\mathbf{0}\}$. So $\ker(\Tilde{\mathbf{N}}_0)\neq\{\mathbf{0}\}$. However, by construction $\Tilde{\mathbf{N}}_0=[\mathbf{N}_0, \mathbf{N}_p\mathbf{e}]$, so $\ker(\Tilde{\mathbf{N}}_0)=\{\mathbf{0}\}$. Thus,
\begin{align}
\forall i\in\mathbb{Z}, \ker(\Tilde{\mathbf{N}}_i)=\{\mathbf{0}\}.
\end{align}

By induction we have $\Tilde{\mathbf{M}}'_i$, $\Tilde{\mathbf{N}}'_i$, $\Tilde{\mathbf{P}}$, $\Tilde{\mathbf{Q}}$ that
\begin{align*}
    \forall i,j,k\in\mathbb{Z}, {\Tilde{\mathbf{M}}'}_i{}^\top\Tilde{\mathbf{N}}'_j={\Tilde{\mathbf{M}}'}_{i+k}{}^\top\Tilde{\mathbf{N}}'_{j+k}, \\
    \forall i,j\in\mathbb{Z}, \Tilde{\mathbf{M}}_i^\top\Tilde{\mathbf{N}}_j=(\Tilde{\mathbf{M}}'_i\Tilde{\mathbf{P}})^\top(\Tilde{\mathbf{N}}'_j\Tilde{\mathbf{Q}}), \\
    \sum_{i=-\infty}^{\infty}\mathrm{im}(\Tilde{\mathbf{M}}'_i)=\mathbb{R}^{\Tilde{l}'}, \sum_{i=-\infty}^{\infty}\mathrm{im}(\Tilde{\mathbf{N}}'_i)=\mathbb{R}^{\Tilde{l}'}, \\
    \forall i\in\mathbb{Z}, \ker(\Tilde{\mathbf{M}}'_i)=\{\mathbf{0}\}, \ker(\Tilde{\mathbf{N}}'_i)=\{\mathbf{0}\}.
\end{align*}

So $\mathbf{M}'_i=\Tilde{\mathbf{M}}'_i$, $\mathbf{N}'_i=\Tilde{\mathbf{N}}'_i$, $\mathbf{P}=\Tilde{\mathbf{P}}$, $\mathbf{Q}=\Tilde{\mathbf{Q}}\mathbf{A}$ satisfies Equation~\ref{eq:lemma2_res_invariance}-\ref{eq:lemma2_res_ker}.

\end{proof}

\begin{proposition_appendix}
Let $\left\{\mathbf{M}_i\right\}_{i=-\infty}^\infty$ be a series of $l\times m$ matrices, $\left\{\mathbf{N}_i\right\}_{i=-\infty}^\infty$ be a series of $l\times n$ matrices.
Then,  $\mathbf{M}_i^\top\mathbf{N}_j$ only depends on $i-j$, if and only if that, there is an integer $l'$, matrices $\mathbf{P}\in\mathbb{R}^{l'\times m}$, $\mathbf{Q}\in\mathbb{R}^{l'\times n}$, and an invertible matrix $\mathbf{A}\in\mathbb{R}^{l'\times l'}$, such that
\begin{align}
    \mathbf{M}_i^\top\mathbf{N}_j=(\mathbf{A}^{-i\top}\mathbf{P})^\top(\mathbf{A}^j\mathbf{Q}),
\end{align}
\end{proposition_appendix}

\begin{proof}

$(\Leftarrow)$ \textit{If} part.

$\mathbf{M}_i^\top\mathbf{N}_j=(\mathbf{A}^{-i\top}\mathbf{P})^\top(\mathbf{A}^j\mathbf{Q})=\mathbf{P}^\top\mathbf{A}^{j-i}\mathbf{Q}$ depends on $i-j$ only.

$(\Rightarrow)$ \textit{Only If} part.

By Lemma~\ref{lemma1} and Lemma~\ref{lemma2}, there is
\begin{align}
    \left\{\mathbf{M}'_i\right\}_{i=-\infty}^\infty\in\mathbb{R}^{l'\times l'}, \\
    \left\{\mathbf{N}'_i\right\}_{i=-\infty}^\infty\in\mathbb{R}^{l'\times l'}, \\
    \mathbf{P}\in\mathbb{R}^{l'\times m}, \mathbf{Q}\in\mathbb{R}^{l'\times n},
\end{align}
such that
\begin{align}
    \label{eq:prop_res_invariance}
    \forall i,j,k\in\mathbb{Z}, {\mathbf{M}'}_i^\top\mathbf{N}'_j={\mathbf{M}'}_{i+k}^\top\mathbf{N}'_{j+k}, \\
    \label{eq:prop_res_equivalence}
    \forall i,j\in\mathbb{Z}, \mathbf{M}_i^\top\mathbf{N}_j=(\mathbf{M}'_i\mathbf{P})^\top(\mathbf{N}'_j\mathbf{Q}), \\
    \label{eq:prop_res_rank}
    \forall i\in\mathbb{Z}, \mathrm{rank}(\mathbf{M}'_i)=l', \mathrm{rank}(\mathbf{N}'_i)=l'.
\end{align}
Since ${\mathbf{M}'}_0^\top\mathbf{N}'_{i-1}={\mathbf{M}'}_1^\top\mathbf{N}'_i$,
\begin{align}
    \mathbf{N}'_i&=({\mathbf{M}'}_1^{-\top}{\mathbf{M}'}_0^\top)\mathbf{N}'_{i-1} \\
    &=({\mathbf{M}'}_1^{-\top}{\mathbf{M}'}_0^\top)^i\mathbf{N}'_0 \\
    &=\mathbf{A}^i\mathbf{N}'_0,
\end{align}
where $\mathbf{A}\in\mathbb{R}^{l'\times l'}$ is an invertible matrix.
Similarly, $\mathbf{M}'_i=\mathbf{B}^i\mathbf{M}'_0$.
Substitute them into Equation~\ref{eq:prop_res_invariance}, we have
\begin{align}
    \begin{split}
    &\forall i,j,k\in\mathbb{Z}, \\
    &{\mathbf{M}'}_0^\top\mathbf{B}^{i\top}\mathbf{A}^j\mathbf{N}'_0={\mathbf{M}'}_0^\top\mathbf{B}^{(i+k)\top}\mathbf{A}^{j+k}\mathbf{N}'_0
    \end{split}
\end{align}
Since $A$, $B$, $\mathbf{N}'_0$ and $\mathbf{M}'_0$ are invertible,
\begin{align}
    \forall ki\in\mathbb{Z}, \mathbf{B}^{k\top}\mathbf{A}^k=\mathbf{I}.
\end{align}
Thus, $\mathbf{B}=\mathbf{A}^{-\top}$.

Thus,
\begin{align}
    \mathbf{M}_i^\top\mathbf{N}_j&=(\mathbf{M}'_i\mathbf{P})^\top(\mathbf{N}'_j\mathbf{Q}) \\
    &=(\mathbf{B}^i\mathbf{M}'_0\mathbf{P})^\top(\mathbf{A}^j\mathbf{N}'_0\mathbf{Q}) \\
    &=(\mathbf{A}^{-i\top}\mathbf{P}')^\top(\mathbf{A}^j\mathbf{Q}'),
\end{align}
where $\mathbf{P}'=\mathbf{M}'_0\mathbf{P}$ and $\mathbf{Q}'=\mathbf{N}'_0\mathbf{Q}$.

\end{proof}

\section{Hyper-parameters for Long-Range Arena}
\label{sec:lra_hyper}
\begin{table}[h]
    \centering
    \begin{tabular}{c||cccc}
    \toprule
        Task & Text & Retrieval & Image & Pathfiner \\
    \midrule
        Batch size &  16 & 32 & 256 & 256 \\
        Epochs & 10 & 10 & 50 & 80 \\
        LR & 1e-5 & 2e-4 & 1e-2 & 5e-4 \\
        Warmup & 4000 & 1000 & 200 & 4000 \\
    \bottomrule
    \end{tabular}
    \caption{Hyper-parameters for Long-Range Arena}
    \label{tab:lra_hyper}
\end{table}

\end{document}